\documentclass[11pt]{article}
\usepackage{bm} 
\usepackage{amsthm}
\usepackage{amsmath}
\usepackage{arydshln}

\usepackage[colorinlistoftodos]{todonotes}

\usepackage{ulem}

\usepackage{color, soul}

\newtheorem{theorem}{Theorem}

\begin{document}
\title{Fisher Information and Natural Gradient Learning of Random Deep Networks}
\author{Shun-ichi Amari\thanks{RIKEN CBS, Wako-shi, Japan} $^{\dag}$
\and Ryo Karakida\thanks{AIST, Tokyo, Japan} \and Masafumi
Oizumi\thanks{Araya Inc., Tokyo, Japan}}
\date{}

\maketitle

\begin{abstract}
A deep neural network is a hierarchical nonlinear model transforming input signals to output signals.  Its input-output relation is considered to be stochastic, being described for a given input by a parameterized conditional probability distribution of outputs.  The space of parameters consisting of weights and biases is a Riemannian manifold, where the metric is defined by the Fisher information matrix. The natural gradient method uses the steepest descent direction in a Riemannian manifold, so it is effective in
learning, avoiding plateaus.  It requires inversion of the Fisher information matrix, however, which is practically impossible when the matrix has a huge number of dimensions.  Many methods for approximating the natural gradient have therefore been introduced. The present paper uses statistical neurodynamical method to reveal the properties of the Fisher information matrix in a net of random connections under the mean field approximation.  We prove that the Fisher information matrix is unit-wise block diagonal supplemented by small order terms of off-block-diagonal elements, which provides a justification for the quasi-diagonal natural gradient method by Y. Ollivier. A unitwise block-diagonal Fisher metrix reduces to the tensor product of the Fisher information matrices of single units. We further prove that the Fisher information matrix of a single unit has a simple reduced form, a sum of a diagonal matrix and a rank 2 matrix of weight-bias correlations.  We obtain the inverse of Fisher information explicitly. We then have an explicit form of the natural gradient, without relying on the numerical matrix inversion, which drastically speeds up stochastic gradient learning. 
\end{abstract}

\section{Introduction}

In modern deep learning, multilayer neural networks are usually trained by using the stochastic gradient-descent method (See Amari, 1967 for one of the earliest proposal of stochastic gradient descent for the purpose of applying multilayer networks). The parameter space of multilayer networks forms a Riemannian space equipped with Fisher information metric. Thus, instead of the usual gradient descent method, the natural gradient or Riemannian gradient method, which takes account of the geometric structure of the Riemmanian space, is more effective  for learning (Amari, 1998).  However, it has been difficult to apply the natural gradient descent because it needs the inversion of the Fisher information matrix, which is computationally heavy. Many approximation methods reducing computational costs have therefore been proposed (see Pascanu \& Bengio, 2013; Grosse \& Martens, 2016; Martens, 2017).

To resolve the computational difficulty of the natural gradient, we analyze the Fisher information matrix of a random network, where the connection weights and biases are randomly assigned, by using the mean field approximation (See also our accompanying paper Amari, Karakida and Oizumi, 2018 for the analysis of feedforward paths).  We prove that, when the number $n$ of neural units in each layer is sufficiently large, the subblocks of the Fisher information matrix ${\bm{G}}$ corresponding to different layers are of order $1/\sqrt{n}$, which is negligibly small.  Thus, ${\bm{G}}$ is approximated by a layer-wise diagonalized matrix. Furthermore, within the same layer, the subblocks among different units are also of order $1/\sqrt{n}$. 

This gives a justification for the approximated natural gradient method proposed by Kurita (1994) and studied in detail by Ollivier (2015) and Marceau-Caron and Ollivier (2016), where the unit-wise diagonalized ${\bm{G}}$ was used.  We further study the Fisher information matrix of a unit ---that is, a simple perceptron--- for the purpose of implementing unit-wise natural gradient learning.  We obtain an explicit form of the Fisher information matrix and its inverse under the assumption that inputs are subject to the standard non-correlated Gaussian distribution with mean 0.  The unit-wise natural gradient is explicitly formulated without numerical matrix inversion, provided inputs signals are subject to independent Gaussian distributions with mean 0, making it possible that natural gradient learning is realized without the burden of heavy computation. The results justify the quasi-diagonal approximation of the Fisher information matrix proposed by Y. Ollivier (2015), although our results are not exactly the same as Ollivier's results. Our approximation method is justified only for random networks under the mean-field assumption. However, it is expected that it would be effective for training actual deep networks considering the good performance shown in Olivier, 2015 and Marceau-Caron and Ollivier, 2016.

The results can be extended to residual deep networks with ReLU.  We show that the inputs to each layer are approximately subject to 0-mean independent Gaussian distributions in the case of a resnet, because of random linear transformations after nonlinear transformations in all layers. Therefore, our method would be particularly effective when residual networks are used.

To understand the structure of the Fisher information matrix, refer to Karakida, Akaho and Amari (2018), which analyzes the characteristics (the distribution of its eigenvalues) of the Fisher information matrix of a random net for the first time.

\section{Deep neural networks}

We consider a deep neural network consisting of $L$ layers.  Let $\stackrel{l-1}{\bm{x}}$ be the input vectors to the $l$-th layer and $\stackrel{l}{\bm{x}}$ the output vector of the $l$-th layer (see Figure 1).  
\begin{figure}
 \centering
 \includegraphics[width=7cm]{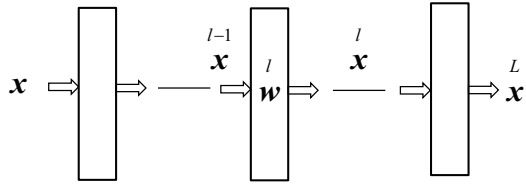}
 \caption{Deep neural network}
\end{figure} 
The input-output relation of the $l$-th layer is written as
\begin{equation}
 \label{eq:am120180618}
 \stackrel{l}{x_i} = \varphi \left(
 \sum_j \stackrel{l}{w_{ij}} \; \stackrel{l-1}{x_j} + \stackrel{l}{b_i}
 \right),
\end{equation}
where $\varphi$ is an activation function such as a rectified linear function (ReLU), sigmoid function, etc. Let $n_l$ be the number of neurons in the $l$-th layer. We assume that $n_1, n_2, \cdots, n_{L-1}$ are large, but the number of neurons in the final layer, $n_L$, can be small. Even $n_L=1$ is allowed.  The weights $\stackrel{l}{w_{ij}}$ and biases $\stackrel{l}{b_j}$ are random variables subject to independent Gaussian distributions with mean 0 and variances $\sigma^2_l/n_{l-1}$ and $\sigma^2_{bl}$, respectively. Note that each weight is a random variable of order $1/ \sqrt{n_{l-1}}$, but the weighted sum $\sum \stackrel{l}{w_{ij}} \stackrel{l-1}{x_j}$ is of order 1.

We recapturate briefly the feedforward analysis of input signals given in Poole et al., 2016 and Amari, Karakida and Oizumi, 2018, to introduce the activity $\stackrel{l}{A}$ and enlargement factor $\stackrel{l}{\chi}$. They also play a role in the feedback analysis obtaining the Fisher information (Schoenholtz et al., 2016; Karakida, Akaho and Amari, 2018).

Let us put
\begin{equation}
 \stackrel{l}{u_i} = \sum_j \stackrel{l}{w_{ij}} \; \stackrel{l-1}{x_j} \;
  + \stackrel{l}{b_i}.
\end{equation}
Given $\stackrel{l-1}{\bm{x}}$, $\stackrel{l}{u_i}$ are independently and identically distributed (iid) Gaussian random variables
with mean 0 and variance
\begin{equation}
 \label{eq:am3}
 \tau^2_l = 
 \frac{\sigma^2_l}n \sum \left( \stackrel{l-1}{x_j} \right)^2 + \sigma^2_{bl}   =
 \stackrel{l-1}{A} \sigma^2_l
  + \sigma^2_{bl},
\end{equation}
where
\begin{equation}
 \stackrel{l-1}A = \frac 1{n_{l-1}} \sum
 {\stackrel{l-1}{x_j}}\; \mbox{}^2
\end{equation}
is the total activity of input $\stackrel{l-1}{\bm{x}}$.

It is easy to show how $\stackrel{l}{A}$ develops across the layers.
Since $x^2_j= \varphi \left(u_j \right)^2$ are iid when
$\stackrel{l-1}{\bm{x}}$ is fixed, the law of large numbers guarantees
that their sum is replaced by the expectation when $n_{l-1}$ is large.
Putting $u_j=\tau_l v$ where is the standard Gaussian
variables, we have a recursive equation,
\begin{equation}
 \stackrel{l}{A} = \int \left\{ \varphi \left( \tau_l \; v \right)\right\}^2 D v,
\end{equation}
where $\tau_l$ in equation (\ref{eq:am3}) depends on $\stackrel{l-1}A$ and
\begin{equation}
 Dv = \frac 1{\sqrt{2 \pi}} \exp
 \left\{ -\frac{v^2}2 \right\} dv.
\end{equation}

Since equation (\ref{eq:am120180618}) gives the transformation from
$\stackrel{l-1}{\bm{x}}$ to $\stackrel{l}{\bm{x}}$, we study how a small
difference $d \stackrel{l-1}{\bm{x}}$ in the input develops to give difference $d \stackrel{l}{\bm{x}}$ in the output.  By differentiating equation (\ref{eq:am120180618}), we have
\begin{equation}
 d \stackrel{l}{\bm{x}} = \stackrel{l}{\bm{B}} d \stackrel{l-1}{\bm{x}}
\end{equation}
where 
\begin{equation}
\stackrel{l}{\bm{B}} = 
 \frac{\partial \stackrel{l}{\bm{x}}}{\partial \stackrel{l-1}{\bm{x}}}
\end{equation}
is a 
matrix whose $(i_l, i_{i-1})$-th element is given by
\begin{equation}
 B^{i_l}_{i_{l-1}} = \varphi' \left(u_{i_l}\right)
 w^{i_l}_{i_{l-1}}.
\end{equation}
It is a random variable of order $1/\sqrt{n_{l-1}}$.  Here and hereafter, we denote $\stackrel{l}{w_{ij}}$ by $w^{i_l}_{i_{l-1}}$, eliminating superfix $l$ and using $i_l$ and $i_{l-1}$ instead of $i$ and $j$.  These index notations are convenient for showing that the corresponding $w$'s belong to layer $l$.

We show how the square of the Euclidean length of $d
 \stackrel{l}{\bm{x}}$,
\begin{equation}
 d {\stackrel{l}{s}}\; \mbox{}^2 = \sum_{l_i} 
  \left( d x_{l_i}\right)^2, 
\end{equation}
is related to that of $d \stackrel{l-1}{\bm{x}}$.  This 
relation can be seen from
\begin{equation}
 \label{eq:am1120180613}
 d \stackrel{l}{s}\mbox{}^2 = \sum_{i_l, i_{l-1}, i'_{l-1}}
  B^{i_l}_{i_{l-1}}
  B^{i_l}_{i'_{l-1}} d  
  x_{i_{l-1}}
  d x_{i'_{l-1}}.
\end{equation}
For any pair $i_{l-1}$ and $i'_{l-1}$, $n_l$ random variables
$B^{i_l}_{i_{l-1}} B^{i_l}_{i'_{l-1}}$ are iid for all $i_l$ when
$\stackrel{l-1}{\bm{x}}$ is fixed, so the law of large numbers
guarantees that
\begin{equation}
 \label{eq:am1220180606}
 \sum_{i_l} B^{i_l}_{i_{l-1}} B^{i_l}_{i'_{l-1}} =
  n_l {\rm{E}} \left[ \varphi' \left(u_{i_l}\right)^2 w^{i_l}_{i_{l-1}}
   w^{i_l}_{i'_{l-1}}\right] + O_p
 \left( \frac 1{\sqrt{n_l}}\right), 
\end{equation}
where ${\rm{E}}$ is the expectation with respect to the weights and
biases and $O_p (1/\sqrt{n})$ represents small terms of stochastic order
$1/\sqrt{n}$.  We use the mean field property that $\varphi'
\left(u_{i_l}\right)$ has the self-averaging property and the average of
the product of $\varphi' \left(u_{i_l}\right)^2$ and $w^{i_l}_{i_{l-1}}
w^{i_l}_{i'_{l-1}}$ in equation (\ref{eq:am1220180606}) splits as
\begin{equation}
 {\mathrm{E}} \left[ \varphi' \left(u_{i_l}\right)^2 \right] 
  {\mathrm{E}} 
 \left[ w^{i_l}_{i_{l-1}} w^{i_l}_{i'_{l-1}}\right].
\end{equation}
This is justified in appendix I.  By putting
\begin{equation}
 \stackrel{l}{\chi} = \sigma^2_l \int \left\{
  \varphi' \left( \tau_l v \right)\right\}^2 D v,
\end{equation}
we have from equation (\ref{eq:am1120180613})
\begin{equation}
 \label{eq:am14}
 d \stackrel{l}{s}\mbox{}^2 = \stackrel{l}{\chi} d \stackrel{l-1}{s^2},
\end{equation}
by using
\begin{equation}
 {\mathrm{E}} \left[
 w^{i_l}_{i_{l-1}} \;  w^{i_l}_{i'_{l-1}}
 \right] =
 \frac{\sigma^2_l}{n_l} \delta_{i_{l-1} i'_{l-1}}.
\end{equation}
Here $\stackrel{l}{\chi}$ which depends on
$\stackrel{l-1}{A}$, is the enlargement factor showing how $d
\stackrel{l-1}{\bm{x}}$ is enlarged or reduced across layer $l$.

From the recursive relation (\ref{eq:am14}), we have
\begin{eqnarray}
 d \stackrel{L}{s^2} &=& \chi^L_l  d \stackrel{l-1}{s^2}, \\
 \chi^L_l &=& \stackrel{L}{\chi} \; \stackrel{L-1}{\chi} \cdots 
 \stackrel{l}{\chi}.
\end{eqnarray}
Assume that all the $\stackrel{l}{\chi}$ are equal.  Then, it gives the
Lyapunov exponent of dynamics equation (\ref{eq:am14}).  When it is larger than
1, the length diverges as the layers proceed, whereas 
when it is smaller than 1, the length decays to 0.  
The dynamics of $d \stackrel{l}{s^2}$ is
chaotic when $\stackrel{l}{\chi} > 1$ (Poole et al, 2016).  Interesting
information processing takes place at the edge of chaos, where
$\chi^L_l$ is nearly equal to 1 (Yang \& Schoenholz, 2017).  We have
interest in the case where $\chi^L_l$ is nearly equal to 1, but each $\stackrel{l}{\chi}$'s are distributed, some being smaller than 1 and the others larger than 1.

\section{Fisher information of deep networks and natural gradient learning}

We study a regression model in which the output of layer $L$, $\stackrel{L}{\bm{x}}= \varphi
(\stackrel{L}{\bm{u}})$,
\begin{equation}
 \label{eq:am1920180731}
 {\bm{y}} = \stackrel{L}{\bm{x}} + {\bm{\varepsilon}},
\end{equation}
where $\bm{\varepsilon} \sim N(0, {\bm{I}})$ is a multivariate Gaussian
random variable with mean 0 and identity covariance matrix ${\bm{I}}$.
Then the probability of ${\bm{y}}$ given input $\bm{x}$ is
\begin{equation}
 p({\bm{y}}|{\bm{x}};{\bm{W}}) = \frac 1{\left(\sqrt{2
					  \pi}\right)^{n_L}}
 \exp \left\{ -\frac 12 \left|{\bm{y}}- \stackrel{L}{\bm{x}} \right|^2 \right\},
\end{equation}
where ${\bm{W}}$ consists of all the parameters $\stackrel{l}{\bm{w}}$,
and $\stackrel{l}{b}$, $l=1, \cdots, L$.  The Fisher information matrix
is given by
\begin{equation}
 \label{eq:am2120180731}
 {\bm{G}} = {\rm{E}}_{\bm{x}, {\bm{y}}} \left[ \left(\partial_{\bm{W}} \log p \right)
  \left(\partial_{\bm{W}} \log p \right)\right],
\end{equation}
where ${\rm{E}}_{{\bm{x}}, {\bm{y}}}$ denotes the expectation with respect to
randomly generated input ${\bm{x}}$ and resultant ${\bm{y}}$ and $\partial_{\bm{W}} = \partial / \partial{\bm{W}}$ is gradient with respect to ${\bm{W}}$.  By using error vector ${\bm{\varepsilon}}$ in (\ref{eq:am1920180731}), we have
\begin{equation}
 \partial_{\bm{W}} \log p =
 {\bm{\varepsilon}} \cdot \partial_{\bm{W}} 
 \stackrel{L}{\bm{x}}.
\end{equation}
For fixed ${\bm{x}}$, expectation with respect to ${\bm{y}}$ is replaced by that of ${\bm{\varepsilon}}$, where ${\rm{E}} \left[ {\bm{\varepsilon}}{\bm{\varepsilon}}\right]= {\bm{I}}$.  Hence, (\ref{eq:am2120180731}) is given by
\begin{equation}
 \label{eq:am2120180508}
 {\bm{G}} = {\rm{E}}_{\bm{x}} \left[ \sum_{i_L} \left\{
 \partial_{\bm{W}} \varphi \left(u_{i_L}\right) \right\}
 \left\{ \partial_{\bm{W}} \varphi \left(u_{i_L}\right)\right\}
 \right].
\end{equation}
Here, we use the dyadic or tensor notation that ${\bm{ab}}$ implies a
matrix $\left(a_i b_j \right)$, instead of vector-matrix notation
${\bm{a}}{\bm{b}}^T$ for column vectors.

Online learning is a method of modifying the current ${\bm{W}}$ such that the current loss
\begin{equation}
 l = \frac 12 \left| {\bm{y}} - {\bm{x}}^L_t \right|^2
\end{equation}
decreases, where $\left({\bm{x}}_t, {\bm{y}}_t \right)$ is the current input-output pair.  The stochastic gradient decent method (proposed in Amari, 1967) uses the gradient of $l$ to modify ${\bm{W}}$,
\begin{equation}
 \Delta {\bm{W}} = -\eta \frac{\partial l}{\partial {\bm{W}}}.
\end{equation}
Historically, the first simulation results applied to four-layer networks for pattern classification were given in a Japanese book (Amari, 1968).  The minibatch method uses the average of $\partial l / \partial{\bm{W}}$ over minibatch samples.  

The negative gradient is a direction to decrease the current loss but is not steepest in a Riemannian manifold.  The true steepest direction is given by
\begin{equation}
 \tilde{\nabla} l = {\bm{G}}^{-1}
  \frac{\partial l}{\partial {\bm{W}}},
\end{equation}
which is called the natural or Riemannian gradient (Amari, 1998).  The natural gradient method is given by
\begin{equation}
 \Delta {\bm{W}} = -\eta \tilde{\nabla} l.
\end{equation}
It is known to be Fisher efficient for estimating ${\bm{W}}$ (Amari, 1998).  Although it gives excellent performances, the inversion of ${\bm{G}}$ is computationally very difficult.

To avoid difficulty, the quasi-diagonal natural gradient method was proposed in Ollivier (2015) and was shown to be very efficient in Marcereau-Caron and Ollivier (2016).  A recent proposal (Ollivier, 2017) looks very promising for realizing natural gradient learning.  The present paper analyzes the structure of the Fisher information matrix.  It will give a justification of the quasi-diagonal natural gradient method.  By using it, we propose a new method of realizing natural gradient learning without the burden of inverting ${\bm{G}}$.

\section{Structure of Fisher information matrix}

To calculate elements of ${\bm{G}}$, we use a new notation combining connection weights ${\bm{w}}$ and bias $b$ into one vector,
\begin{equation}
 \stackrel{l}{\bm{w}^{\ast}} = \left( \stackrel{l}{\bm{w}},
				\stackrel{l}{b}\right). 
\end{equation}
For the $i_{l}$-th unit of layer $l$, it is
\begin{equation}
 {\bm{w}}^{\ast}_{i_l} = \left(
 w^{i_l}_{i_{l-1}}, b^{i_l} \right).
\end{equation}
For $l>m$, we have the recursive relation
\begin{equation}
 \frac{\partial \stackrel{l}{\bm{x}}}{\partial
  \stackrel{m}{\bm{w}^{\ast}}}
 = \frac{\partial \stackrel{l}{\bm{x}}}{\partial \stackrel{l-1}{\bm{x}}} \; 
  \frac{\partial \stackrel{l-1}{\bm{x}}}{\partial
  \stackrel{m}{\bm{w}^{\ast}}} = \stackrel{l}{\bm{B}}
 \frac{\partial \stackrel{l-1}{\bm{x}}}{\partial \stackrel{m}{\bm{w}^{\ast}}}.
\end{equation}
Starting from $l= L$ and using 
\begin{equation}
 \frac{\partial \stackrel{m}{\bm{x}}}{\partial
  \stackrel{m}{\bm{w}^{\ast}}}
  = \varphi' \left(\stackrel{m}{\bm{u}}\right)
  \stackrel{m-1}{\bm{x}},
\end{equation}
we have
\begin{equation}
 \frac{\partial \stackrel{L}{\bm{x}}}{\partial
  \stackrel{m}{\bm{w}^{\ast}}} = \stackrel{L}{\bm{B}} \cdots
 \stackrel{m+1}{\bm{B}} \varphi' \left( \stackrel{m}{\bm{u}}\right)
 \stackrel{m-1}{\bm{x}}.
\end{equation}
Put
\begin{equation}
 {\bm{B}}^L_{m+1} = \stackrel{L}{\bm{B}} \cdots
 \stackrel{m+1}{\bm{B}},
\end{equation}
which is a product of $L-(m-1)$ matrices.  The elements of ${\bm{B}}^L_{m+1}$ are denoted by $B^{i_L}_{i_m}$.  

We calculate the Fisher information ${\bm{G}}$ given in
equation (\ref{eq:am2120180508}).  The elements of ${\bm{G}}$
with respect to layers $l$ and $m$ are written as
\begin{equation}
 \label{eq:am3020180731}
 {\bm{G}} \left( \stackrel{m}{\bm{w}^{\ast}},
	   \stackrel{l}{\bm{w}^{\ast}}\right) =
 {\rm{E}}_{\bm{x}} \left[ \frac{\partial \stackrel{L}{\bm{x}}}{\partial
	   \stackrel{m}{\bm{w}^{\ast}}} \cdot
 \frac{\partial \stackrel{L}{\bm{x}}}{\partial \stackrel{l}{\bm{w}^{\ast}}}
  \right],
\end{equation}
where $\cdot$ denotes the innor product with respect to
$\stackrel{L}{\bm{x}}$. The $\left(i_m, i_{m-1} \right)$ emelents of $\partial \stackrel{L}{\bm{x}} / \partial \stackrel{m}{{\bm{w}}^{\ast}}$ are, for fixed $i_L$,
\begin{equation}
 B^{i_L}_{i_m} \varphi' \left( u_{i_m} \right)
  x_{i_{m-1}}.
\end{equation}
Hence, (\ref{eq:am3020180731}) is written in the component form as
\begin{equation}
 \left[ {\bm{G}} \left( \stackrel{m}{\bm{w}^{\ast}}, \stackrel{l}{{\bm{w}}^{\ast}}\right) \right]^{i_m i_l}_{i_{m-1} i_{l-1}} =
  \sum_{i_L} B^{i_L}_{i_m} B^{i_L}_{i_l} \varphi'
  \left( u_{i_m} \right) \varphi' 
  \left( u_{i_l} \right) x_{i_{m-1}} x_{i_{l-1}}.
\end{equation}
We first consider the case $m=l$, that is, two neurons are in the same layer $m$.  The following lemma is usuful for evaluating $\sum B^{i_L}_{i_m} B^{i_L}_{i'_m}$.

{\textbf{Domino Lemma}} \quad  We assume that all $n_l$ are of order $n$.
\begin{equation}
 \label{eq:am28}
 \sum_{i_L, i'_L} \delta_{i_L i'_L}  B^{i_L}_{i_m}
  B^{i'_L}_{i'_m} =
 \chi^L_{m+1} \delta_{i_m i'_m} + O_p
 \left( \frac 1{\sqrt{n}}\right).
\end{equation}

\begin{proof}
We first prove the case with $m=L-1$.  We have
\begin{equation}
 \label{eq:am29}
 \sum \delta_{i_L i'_L} B^{i_L}_{i_{L-1}} B^{i'_L}_{i'_{L-1}} =
 \sum_{i_L} \left\{ \varphi' \left(u_{i_L}\right)\right\}^2
 w^{i_L}_{i_{L-1}} w^{i_L}_{i'_{L-1}}.
\end{equation}
When $i_{L-1}= i'_{L-1}$, this is a sum of $n_L$ iid random variables
$\left\{ \varphi' \left(u_{i_L}\right)\right\}^2
\left(w^{i_L}_{i_{L-1}}\right)^2$, when input ${\bm{x}}$ is fixed.
Therefore, the law of large numbers guarantees that, as $n_L$ goes to
infinity, 
their sum
converges to the expectation,
\begin{equation}
 n_L {\rm{E}}_{\bm{x}} \left[ \left\{ \varphi' \left(u_{i_L}\right)\right\}^2
  \left(w^{i_L}_{i_{L-1}}\right)^2 \right] = \stackrel{L}{\chi}　\label{MFA}
\end{equation}
under the mean field approximation for any $i_L$.  For fixed $i_{L-1} \ne i'_{L-1}$,
the right-hand side of equation (\ref{eq:am29}) is also a sum of iid variables with mean 0.  Hence, its
mean is 0.  We evaluate its variance, proving that the variance is
\begin{equation}
 n_L {\rm{E}} \left[ \left\{ \varphi' \left(u_{i_L}\right)\right\}^4
 \left(w^{i_L}_{i_{L-1}}\right)^2 
  \left(w^{i_L}_{i'_{L-1}}\right)^2
 \right]
\end{equation}
which is of order $1/n_L$, because
${\rm{E}}\left[\left(w^{i_L}_{i_{L-1}}\right)^2\right]$ is of order
$1/n_L$.  Hence we have
\begin{equation}
 \label{eq:am32}
  \sum \delta_{i_L i'_L} B^{i_L}_{i_{L-1}} B^{i'_L}_{i'_{L-1}} =
 \stackrel{L}{\chi} \delta_{i_{L-1} i'_{L-1}} + O_p
 \left( \frac 1{\sqrt{n_L}}\right).
\end{equation}
When $m<L-1$, we repeat the process $L-1, \cdots$.  Then $\delta_{i_L i'_L}$ in the left-hand side of equation (\ref{eq:am28}) propagates to give $\delta_{i_m i'_m}$ like the domino effect, leaving multiplicative factors $\stackrel{l}{\chi}$. This proves the theorem.
\end{proof}

\begin{flushleft}
{\textbf{Remark:}} The domino lemma holds irrespective of $n_l>n_{l-1}$
 or $n_l < n_{l-1}$, provided they are large.  However, matrix
 ${\bm{B}}^L_m$ is not of full rank, its rank being $\min \left\{n_L,
 \cdots, n_m \right\}$.
\end{flushleft}

By using this result, we evaluate off-diagonal
blocks of ${\bm G}$ under the mean field approximation (\ref{MFA}).

\begin{theorem}\upshape
The Fisher information matrix ${\bm{G}}$ is unit-wise diagonal except
for terms of stochastic order $O_p(1/\sqrt{n})$.
\end{theorem}

\begin{proof}
We first calculate the off-diagonal blocks of the Fisher information matrix within the same layers.  The Fisher information submatrix within layer $m$ is
\begin{equation}
 {\bm{G}} \left( \stackrel{m}{{\bm{w}}^{\ast}},  \stackrel{m}{{\bm{w}}^{' \ast}} \right) =
 E_{\bm{x}} \left[ \sum \delta_{i_L i'_L}  B^{i_L}_{i_m}
  B^{i_L}_{i'_m} \varphi' \left(u_{i_m}\right)
 \varphi' \left(u_{i'_m}\right) x_{i_{m-1}} x_{i'_{m-1}}
 \right],
\end{equation}
which are elements of submatrix of ${\bm{G}}$ corresponding to neurons $i_m$ and $i'_m$ both in the same layer $m$.  By the domino lemma, we have
\begin{equation}
 G \left( \stackrel{m}{{\bm{w}}^{\ast}}, \stackrel{m}{{\bm{w}}^{' \ast}}  \right) =
 {\rm{E}}_{\bm{x}} \left[ \chi^L_m 
  \left\{ \varphi' \left(u_{i_m}\right)\right\}^2
 x_{i_{m-1}} x_{i'_{m-1}}\right] \delta_{i_m i'_m} +
 O_p \left(\frac 1{\sqrt{n}}\right).
\end{equation}
This shows that 
the submatrix
is unit-wise block diagonal: that is, the blocks of
 different neurons $i_m$ and $i'_m \; \left(i_m \ne i'_m \right)$ are 0
 except for terms of order $1/\sqrt{n}$.

We next study the blocks of different layers $l$ and $m \; (m<l)$,
\begin{equation}
 {\bm{G}}\left( \stackrel{\ast}{{\bm{w}}^l},  
 \stackrel{\ast}{{\bm{w}}^m} \right) =
 {\rm{E}}_{\bm{x}} \left[ \sum_{i_L, i'_L} \delta_{i_L i'_L}
  B^{i_L}_{i_l}
   B^{i'_L}_{i_m} \varphi'
 \left(u_{i_l}\right) \varphi' \left(u_{i_m}\right)
 \stackrel{l-1}{\bm{x}} \; \stackrel{m-1}{\bm{x}}\right].
\end{equation}
We have
\begin{equation}
  B^{i_L}_{i_m} = \sum_{i_l}  B^{i_L}_{i_l}
  B^{i_l}_{i_m}.
\end{equation}
By using the domino lemma, ${\bm{G}}$ is written as
\begin{equation}
 \label{eq:am4220180613}
 {\rm{E}}_{\bm{x}} \left[ \chi^L_l B^{i_l}_{i_m} \varphi'
 \left(u_{i_l} \right) \varphi' \left(u_{i_m} \right)
 \stackrel{l-1}{\bm{x}} \; \stackrel{m-1}{\bm{x}} \right].
\end{equation}
When $m=l-1$,
\begin{equation}
 B^{i_l}_{i_{m-1}} = \varphi' \left(u_{i_l}\right)
 w^{i_l}_{i_{m-1}}
\end{equation}
and hence it is of order $1/\sqrt{n}$.  In general,
$B^{i_l}_{i_m}$ is a sum of $n^{l-m}\; 0$ mean iid random
variables with variance of order $1/n^{l-m+1}$.  Hence, its mean is 0 and
variance is of order $1/n$, proving that (\ref{eq:am4220180613}) is of
 order $1/\sqrt{n}$.
\end{proof}

Inspired from this, we define a new metric ${\bm{G}}^{\ast}$ as an
approximation of ${\bm{G}}$, such that all the off-diagonal block terms
of ${\bm{G}}$ are discarded, putting them equal to 0.  We study the
natural (Riemannian) gradient method which uses ${\bm{G}}^{\ast}$ as the
Riemannian metric.  Note that ${\bm{G}}^{\ast}$ is an approximation of
${\bm{G}}$, ${\bm{G}}$ tending to ${\bm{G}}^{\ast}$ for $n \rightarrow
\infty$ in the max-norm, but ${\bm{G}}^{\ast-1}$ is not a good approximation to ${\bm{G}}^{\ast}$.  This is because the max-norm of a matrix is not sub-multiplicative.  See the remark below.

{\textbf{Remark:}} \quad One should note that the approximately block
diagonal structure is not closed in the matrix multiplication and
inversion.  Even though ${\bm{G}}$ is approximately unitwise block
diagonal, its square is not, as is shown in the following.  For
simplicity, we assume that
\begin{equation}
 {\bm{G}} = {\bm{I}} + \frac 1{\sqrt{n}}{\bm{B}},
\end{equation}
where $\bm{I}$ is an identity matrix and ${\bm{B}}= \left(b_{ij}\right)$ is a random matrix of order 1,
$b_{ij}$ being independent random variables subject to $N(0, 1)$.  Then 
\begin{equation}
 {\bm{G}}^2 = {\bm{I}} + \frac{2{\bm{B}}}{\sqrt{n}} +
 \frac 1n {\bm{B}}^2.
\end{equation}
Here the $(i, j)$-th element of ${\bm{B}}^2$ is
\begin{equation}
 \sum_k B_{ik} B_{kj},
\end{equation}
a sum of $n$ independent random variables.  Hence, although its mean is
0, it is of order 1.  Hence, the off-diagonal elements are no more
small.  The same situation holds for ${\bm{G}}^{-1}$.

We may also note that the Riemannian magnitude of vector ${\bm{a}}$,
\begin{equation}
 {\bm{a}}^T {\bm{G}}{\bm{a}} = \sum G_{ij} a_i a_j
\end{equation}
is not approximated by ${\bm{a}}^T {\bm{G}}^{\ast}{\bm{a}}$, because we
cannot neglect the off-diagonal elements of ${\bm{G}}$.

Recently, Karakida, Akaho $\&$ Amari (2018) analyzed characteristics of the original metric ${\bm{G}}$  (not ${\bm{G}}^{\ast}$).  They evaluated the traces of ${\bm{G}}$ and ${\bm{G}}^2$ to analyze the distribution of eigenvalues of ${\bm{G}}$, which proves that the small off-diagonal elements cause a long-tail distribution of eigenvalues. 
This elucidates the landscape of the error surface in a random deep net.  In contrast, the present study focuses on the approximated metric ${\bm{G}}^*$.  It enables us to give an explicit form of the Fisher information matrix, directly applicable to natural gradient methods, as follows.

\section{Unit-wise Fisher information}

Because ${\bm{G}}^{\ast}$ is unit-wise block-diagonal, it is enough to
calculate the Fisher information matrices of single units.  We assume
that its input vector ${\bm{x}}$ is subject to $N(0, {\bm{I}})$.  This does not hold in general.  However, it holds approximately for a randomly connected resnet, as is shown in the next section. 

Let us
introduce a new $(n + 1)$-dimensional vectors for a single unit:
\begin{eqnarray}
 {\bm{w}}^{\ast} &=& ({\bm{w}}, w_0), \\
 {\bm{x}}^{\ast} &=& ({\bm{x}}, x_0),
\end{eqnarray}
where $w_0=b$ and $x_0=1$.  Then, the output of the unit is $\varphi (u)
= \varphi({\bm{w}}^{\ast} \cdot {\bm{x}}^{\ast})$, $u={\bm{w}}^{\ast} \cdot {\bm{x}}^{\ast}$.  The Fisher information matrix is an $(n+1)\times (n+1)$ matrix written as
\begin{equation}
 {\bm{G}} = {\rm{E}}_{\bm{x}} \left[ 
 \left\{ \varphi'(u)\right\}^2 
  {\bm{x}}^{\ast} {\bm{x}}^{\ast}
 \right].
\end{equation}

We introduce a set of new $n+1$ orthonormal basis vectors in the space
of ${\bm{x}}^{\ast} = \left({\bm{x}}, x_0 \right)$ as
\begin{eqnarray}
 {\bm{e}}^{\ast}_0 &=& (0, \cdots, 0, 1), \\
 {\bm{e}}^{\ast}_i &=& ({\bm{a}}_i, 0), \quad
  i=1, 2, \cdots, n-1, \\
 {\bm{e}}^{\ast}_n &=& \frac 1w ({\bm{w}}, 0), \quad
  w^2 = {\bm{w}} \cdot {\bm{w}},
\end{eqnarray}
where ${\bm{a}}_i, i=1, \cdots n$, are arbitrary orthogonal unit vectors, satisfying ${\bm{a}}_i
\cdot {\bm{w}}=0$, ${\bm{a}}_i \cdot {\bm{a}}_j = \delta_{ij}$ .
That is, $\left\{ {\bm{e}}^{\ast}_1, \cdots, {\bm{e}}^{\ast}_n \right\}$ is a rotation of $\left\{ {\bm{e}}_1, \cdots, {\bm{e}}_n \right\}$ and we put ${\bm{e}}_0={\bm{e}}^{\ast}_0$.

Here $\left\{{\bm{e}}^{\ast}_i \right\}, i=0, 1, \cdots, n, n+1$ are
mutually orthogonal unit vectors and ${\bm{e}}^{\ast}_n$ is the unit
vector in the direction of ${\bm{w}}$.  Since ${\bm{x}}^{\ast}$ and
${\bm{w}}^{\ast}$ are represented in the new basis as
\begin{equation}
 {\bm{x}}^{\ast} = \sum^n_{i=0} x^{\ast}_i {\bm{e}}^{\ast}_i, \quad
 {\bm{w}}^{\ast} = b {\bm{e}}^{\ast}_0 + w {\bm{e}}^{\ast}_n,
\end{equation}
we have
\begin{equation}
 \label{eq:am5520180618}
 {\bm{G}} = {\rm{E}} \left[
  \left\{ \varphi' \left( {\bm{w}}^{\ast} \cdot
		    {\bm{x}}^{\ast} \right)\right\}^2 
  {\bm{x}}^{\ast}{\bm{x}}^{\ast}
 \right].
\end{equation}
Moreover, $\left(x^{\ast}_1, \cdots, x^{\ast}_n \right)$ are orthogonal
transformation of ${\bm{x}}= \left(x_1, \cdots, x_n\right)$.  Hence,
$x^{\ast}_i, i=1, \cdots, n$, are jointly independent Gaussian, subject
to $N(0, 1)$, and $x^{\ast}_0=1$.

In order to obtain ${\bm{G}}$, let us put 
\begin{equation}
 {\bm{G}} = \sum^n_{i, j=0} A_{ij}{\bm{e}}^{\ast}_i
  {\bm{e}}^{\ast}_j
\end{equation}
in the dyadic notation.  Then, the coefficients $A_{ij}$ are given by
\begin{equation}
 A_{ij}={\bm{e}}^{\ast}_i {\bm{G}}{\bm{e}}^{\ast}_j,
\end{equation}
which are elements of ${\bm{G}}$ in the coordinate system $\left\{ {\bm{e}}^{\ast}_i \right\}$.  From ${\bm{w}}^{\ast} \cdot {\bm{x}}^{\ast}=wx_n+w_0$ and
equation (\ref{eq:am5520180618}), we have
\begin{eqnarray}
 \label{eq:am5020180511}
 A_{00} &=& {\bm{e}}^{\ast}_0 {\bm{G}} {\bm{e}}^{\ast}_0 = 
 \int \left\{ \varphi' \left(w {x}^{\ast}_n + w_0 \right)\right\}^2
  D x^{\ast}_n, \\
 \label{eq:am5120180511}
 A_{0n} &=& {\bm{e}}^{\ast}_0 {\bm{G}} {\bm{e}}^{\ast}_n =
 \int {x}^{\ast}_n \left\{ \varphi' \left(w {x}^{\ast}_n + w_0
				   \right)\right\}^2 D {x}^{\ast}_n, \\
 \label{eq:am5220180511}
 A_{nn} &=& {\bm{e}}^{\ast}_n {\bm{G}} {\bm{e}}^{\ast}_n = 
 \int {x}^{\ast 2}_n \left\{ \varphi' \left(w {x}^{\ast}_n + w_0
		     \right)\right\}^2 D {x}^{\ast}_n,
\end{eqnarray}
which depend on $\left({\bm{w}}, w_0 \right)$.  We further have, for
$i=1, \cdots, n-1$,
\begin{eqnarray}
  A_{ii} &=& {\bm{e}}^{\ast}_i {\bm{G}} {\bm{e}}^{\ast}_i = A_{00}, \\
  A_{ij} &=& {\bm{e}}^{\ast}_i {\bm{G}} {\bm{e}}^{\ast}_j = 0 \quad (j \ne i), \\
  A_{i0} &=& {\bm{e}}^{\ast}_i {\bm{G}} {\bm{e}}^{\ast}_0 = 0 \quad (i \ne n).
\end{eqnarray}

From these, we obtain ${\bm{G}}$ in the dyadic form
\begin{eqnarray}
 \label{eq:am5520180509}
 {\bm{G}} &=& A_{00} \sum^n_{i=0} {\bm{e}}^{\ast}_i {\bm{e}}^{\ast}_i
  + \left(A_{nn}-A_{00}\right) {\bm{e}}^{\ast}_n {\bm{e}}^{\ast}_n 
  \nonumber \\
 &&\qquad\quad\  + A_{0n} \left( {\bm{e}}^{\ast}_0 {\bm{e}}^{\ast}_n
   + {\bm{e}}^{\ast}_n {\bm{e}}^{\ast}_0 \right).
\end{eqnarray}
The elements of ${\bm{G}}$ in the basis $\left\{ {\bm{e}}^{\ast}_1,
\cdots, {\bm{e}}^{\ast}_n, {\bm{e}}^{\ast}_0 \right\}$ are 
\begin{equation}
 \label{eq:am5620180509}
 {\bm{G}} = \left[
 \begin{array}{c:c}
 \begin{matrix}
  A_{00} & \mbox{}& 0  \\
  \mbox{} & \ddots & \mbox{} \\
  0 & \mbox{} & A_{00} 
 \end{matrix} & \text{\Large{\:0}} \\
  \hdashline %
 \text{\rule{0pt}{17pt}\Large{0}}  &
 \begin{matrix}
 A_{nn} & A_{n0}  \\
 A_{n0} & A_{00}.
 \end{matrix}
 \end{array}
  \right],
\end{equation}
which shows that ${\bm{G}}$ is a sum of a diagonal matrix and a rank 2 matrix.

The inverse of ${\bm{G}}$ has the same block form as
equations (\ref{eq:am5520180509}) and (\ref{eq:am5620180509}).  Note that
\begin{eqnarray}
 \sum {\bm{e}}^{\ast}_i {\bm{e}}^{\ast}_i &=& {\bm{I}}, \\
 {\bm{e}}^{\ast}_n {\bm{e}}^{\ast}_n &=& \frac 1{w^2}{\bm{w}}{\bm{w}},
  \\
 \label{eq:am5920180511}
 {\bm{e}}^{\ast}_0 {\bm{e}}^{\ast}_n + {\bm{e}}^{\ast}_n {\bm{e}}^{\ast}_0
  &=& \frac{1}{w}
 \left[
\newlength{\myheighta}
\setlength{\myheighta}{2cm}
 \begin{array}{@{\,}ccc:c@{\,}}
 \parbox[c][\myheighta][c]{0cm}{} \mbox{} & 0 & \mbox{} & {\bm{w}} \\
  \hdashline %
 \mbox{} & {\bm{w}}^T & \mbox{} & 0
  \end{array}
 \right] 
 = \frac 1w \left({\bm{e}}_0 \tilde{\bm{w}} + \tilde{\bm{w}} {\bm{e}}_0  \right), \\
 {\bm{e}}^{\ast}_0 {\bm{e}}^{\ast}_0 &=& 
  \left[
\newlength{\myheight}
\setlength{\myheight}{2cm}
  \begin{array}{@{\,}ccc:c@{\,}}
   \parbox[c][\myheight][c]{0cm}{} \mbox{} &0 & \mbox{} & 0 \\
   \hdashline %
   \mbox{}& 0 & \mbox{} & 1
  \end{array}
 \right],
\end{eqnarray}
where $\tilde{\bm{w}}=({\bm{w}}, 0)$.

By using these relations, ${\bm{G}}$ is expressed in the original basis as
\begin{equation}
 {\bm{G}} = \sum A_{00} \; {\bm{I}} +
  \frac{\left(A_{nn}-A_{00} \right)}{w^2} 
  \tilde{\bm{w}} \tilde{\bm{w}} +
  \frac{A_{0n}}{w} \left( {\bm{e}}_0 \tilde{\bm{w}} +
   \tilde{\bm{w}} {\bm{e}}_0 \right).
\end{equation}
The inverse of ${\bm{G}}$ has also the same form, so we have an explicit form of ${\bm{G}}^{-1}$
\begin{eqnarray}
 {\bm{G}}^{-1} &=& \bar{A}_{00} {\bm{I}} + 
 \frac{X}{w^2} \tilde{\bm{w}} \tilde{\bm{w}} +
  \frac{Y}{w} \left(
  {\bm{e}}^{\ast}_0 \tilde{\bm{w}} + \tilde{\bm{w}}{\bm{e}}^{\ast}_0 \right) \\
  &&\qquad\quad + Z{\bm{e}}^{\ast}_0 {\bm{e}}^{\ast}_0,
\end{eqnarray}
where 
\begin{eqnarray}
 \label{eq:am7020180619}
 \bar{A}_{00} &=& \frac 1{A_{00}}, \\
 \label{eq:am7120180619}
 X &=& \frac 1D A_{00} -\bar{A}_{00}, \quad
 Y= \frac{-A_{n0}}D, \quad Z= \frac{A_{nn}}D-\bar{A}_{00}, \\
 \label{eq:am7220180619}
 D &=& A_{00} A_{nn}-A^2_{n0}.
\end{eqnarray}
By using the above equations, ${\bm{G}}^{-1}{\bm{x}}^{\ast}$ is obtained
explicitly, so we do not need to calculate back-propagated ${\bm{G}}$
and its inverse for the natural gradient update of ${\bm{W}}$.

The natural gradient method for each unit is written by using the
back-propagated error $e$ as
\begin{equation}
 \Delta {\bm{w}}^{\ast} = -\eta e {\bm{G}}^{-1}{\bm{x}}^{\ast},
\end{equation}
which splits as
\begin{eqnarray}
 \label{eq:am6720180509}
 \Delta {\bm{w}} &=& -\eta e \left[
 \bar{A}_{00}{\bm{x}} + \left( \frac{X}{{\bm{w}}^2}
 {\bm{w}} \cdot {\bm{x}}+ \frac Yw \right) {\bm{w}}
 \right], \\
 \label{eq:am6820180509}
 \Delta w_0 &=& -\eta e \left(\bar{A}_{00} + 
  \frac{{\bm{w}} \cdot {\bm{x}}}w Y + Z \right) w_0.
\end{eqnarray}
The back-propergated error $e$ is calculated as follows. 
Let $e_{i_m}$ be the back-propagated error of neuron $i_m$ in layer $m$.
It is given by the well-known error backpropagation as
\begin{eqnarray}
 \label{eq:am7620180619}
  e_{i_m} &=&  
  \sum_{i_L, i_{m+1}} \stackrel{L}{e_{i_L}}
  B^{i_L}_{i_L, i_{m+1}}
 \varphi' \left( u_{i_{m+1}} \right)  
  w^{i_{m+1}}_{i_m},
  \\
 \stackrel{L}{\bm{e}} &=& {\bm{y}} - \stackrel{L}{\bm{x}}.
\end{eqnarray}

We can implement the unit-wise natural gradient method
using equations (\ref{eq:am6720180509}) and (\ref{eq:am6820180509}) without calculating ${\bm{G}}^{\ast-1}$.  However, the unit-wise ${\bm{G}}$ is derived under the condition that the input ${\bm{x}}$ to each neuron
is subject to a 0-mean Gaussian distribution.  This does not hold in general, so we need to adjust ${\bm{x}}$ by a linear transformation.  We will see that a residual network automatically makes the input $\bm{x}$ to each layer be subject to a 0-mean Gaussian distribution.

Obviously, ${\bm{W}}$ is no more random Gaussian with mean 0 after
learning.  However, since the unit-wise natural gradient proposed
here is computationally so easy, it is worth trying for practical
applications even after learning.

Except for a rank 1 term $\tilde{\bm{w}} \tilde{\bm{w}}= \left(w_i
w_j\right)$ and bias terms ${\bm{e}}_0 \tilde{\bm{w}}+ \tilde{\bm{w}}{\bm{e}}_0$, ${\bm{G}}^{-1}$ is a diagonal matrix.  In other words, as
is seen in equation (\ref{eq:am5920180511}), it is diagonal except for a raw and column corresponding to the bias terms and the rank 1
term $\tilde{\bm{w}} \tilde{\bm{w}}$.  Except for the rank 1 term $\tilde{\bm{w}} \tilde{\bm{w}}$, it has the same structure as that of the quasi-diagonal matrix of Ollivier (2015), justifying the
quasi-diagonal method.

\section{Fisher information of residual network}

The residual network has direct paths from its input to output in each
layer.  We treat the following block of layer $l$: The layer $l$
transforms input $\stackrel{l-1}{\bm{x}}$ to output
$\stackrel{l}{\bm{x}}$ by
\begin{eqnarray}
 \stackrel{l}{x_i} &=& \sum_j \stackrel{l}{v_{ij}}\varphi
  \left(\stackrel{l}{u_j}\right) + \alpha \stackrel{l-1}{x_i}, \\
 \stackrel{l}{u_j} &=& \sum_k \stackrel{l}{w_{jk}} \stackrel{l-1}{x_k}
  + \stackrel{l}{b_j}
\end{eqnarray}
(see Figure 2).  
\begin{figure}
 \centering
 \includegraphics[width=7cm]{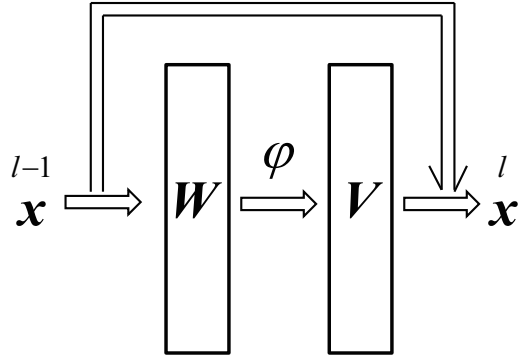}
 \caption{Residual network}
\end{figure} 
Here $\alpha \le 1$ is a decay factor, ($\alpha=1$
is conventionally used), and $\stackrel{l}{v_{ij}}$ are randomly generated iid Gaussian variables subject to $N(0, \sigma^2_v/n)$.

We show how the activity develops in a residual network (Yang and Schoenholtz, 2017).  We easily have the recursive relation,
\begin{eqnarray}
 \stackrel{l}{A} &=& \frac 1n \sum \left(v_{ij} \varphi
 \left( \stackrel{l}{u_j}\right)+ \alpha \stackrel{l-1}{x_i}\right)
 \left( v_{ik} \varphi \left( \stackrel{l}{u_k}\right)+
  \alpha \stackrel{l-1}{x_i}\right) \\
   \label{eq:8820170801}
 &=& \sigma^2_v \bar{A}^l + \alpha^2 \stackrel{l-1}{A},
\end{eqnarray}
where
\begin{equation}
 \bar{A}^l = \int \left\{ \varphi \left(\tau_l v \right)\right\}^2 Dv
\end{equation}
Eq (\ref{eq:8820170801}) shows that $\stackrel{l}{A}$ diverges to infinity as $l$ increase when $\alpha \ge 1$. Therefore, we recommend to use $\alpha < 1$.

The layer $l$ consists of two sublayers.  One is the ordinary neural
network with weights $\stackrel{l}{w_{jk}}$, bias $\stackrel{l}{b_j}$
and activation function $\varphi$.  The other is a linear network that
randomizes the outputs $\varphi \left(u_j \right)$ of the first layer,
transforming them to asymptotically independent 0-mean Gaussian random
variables.  Therefore, mean 0 quasi independent Gaussianity is guaranteed for
$\stackrel{l}{\bm{x}}$.  Since the second linear network is used for the
purpose making output $\stackrel{l}{\bm{x}}$ subject to 0-mean 
independent Gaussian distributions, we fix them throughout the learning
process for simplicity.  That is, $\left\{ \stackrel{l}{w_{ij}},
\stackrel{l}{b_i}\right\}$ are subject only to stochastic gradient
learning.  Therefore, we study the Fisher information with respect to
$\left\{\stackrel{l}{w_{ij}}, \stackrel{l}{b_i}\right\}$ only.  It is redundant to train both $v_{ij}$ and $w_{ij}$.  The role of $v_{ij}$ is to Gaussianize the outputs of layers.  We recommend to fix $\stackrel{l}{v_{ij}}$ throughout learning processes once they are randomly assigned in the initial stage.  

We calculate the following recursive formula,
\begin{eqnarray}
 \frac{\partial \stackrel{l}{x_i}}{\partial \stackrel{m}{w_{st}}} &=&
  \sum_{j, k} \stackrel{l}{v_{ij}} \varphi'
  \left(\stackrel{l}{u_j}\right)
  w^l_{jk} \frac{\partial \stackrel{l-1}{x_k}}{\partial
  \stackrel{m}{w_{st}}} + \alpha
 \frac{\partial \stackrel{l-1}{x_i}}{\partial \stackrel{m}{w_{st}}} \\
 &=& \sum_k \stackrel{l}{B_{ik}}
  \frac{\partial \stackrel{l-1}{x_k}}{\partial \stackrel{m}{w_{st}}}, 
\end{eqnarray}
where 
\begin{equation}
 \stackrel{l}{B_{ik}} = \sum_j \stackrel{l}{v_{ij}} \varphi'
 \left(\stackrel{l}{u_j}\right) \stackrel{l}{w_{jk}}
  + \alpha \delta_{ik}
\end{equation}
in the case of a residual net.  Note that $\stackrel{l}{B_{ik}}$ is of
order $1/\sqrt{n}$ when $i \ne k$, and
\begin{equation}
 \stackrel{l}{B_{ii}} = \alpha + O_p (1/\sqrt{n}).
\end{equation}
From this we have
\begin{equation}
 \frac{\partial \stackrel{L}{x_i}}{\partial \stackrel{m}{w_{st}}} =
  \sum_k B^i_k \stackrel{m}{v_{ks}} \varphi'
 \left(\stackrel{m}{u_s}\right) \stackrel{m-1}{x_t}
\end{equation}
and
\begin{equation}
 G \left( \stackrel{m}{w_{st}}, \stackrel{l}{w_{s't'}}\right) =
 E_{\bm{x}} \left[ \sum_i  B^i_{k_m}
  B^i_{k_l} \stackrel{m}{v_{k_m s}} \stackrel{l}{v_{k_l s'}}
  \varphi' \left(\stackrel{m}{u_s}\right)
  \varphi' \left(\stackrel{l}{u_{s'}}\right) \stackrel{m-1}{x_t}
  \stackrel{l-1}{x_{t'}}
\right].
\end{equation}
Here we again use the domino theorem, where previous $\chi$ is replaced
by
\begin{equation}
 \bar{\chi} = \sigma^2_v \chi + \alpha.
\end{equation}
Since $\stackrel{m}{v_{ks}}$ and $\stackrel{l}{v_{k's'}}$ are
independent when $m \ne l$, $G \left(\stackrel{m}{w_{st}},
\stackrel{l}{w_{s' t'}}\right)$ is of order $1/\sqrt{n}$.  This is true
when $m=l$, $s \ne s'$.  So we have the following theorem.

\begin{theorem}\upshape
The Fisher information matrix ${\bm{G}}$ of a residual net is unit-wise diagonal to within terms of order $1/\sqrt{n}$.

Since $\stackrel{l}{\bm{x}}$ are asymptotically subject to $N(0,
 \bar{\sigma}^2)$, where $\bar{\sigma}^2$ is determined from $\chi$, we
 can apply our procedure of unit-wise natural gradient leaning described
 in the previous section without any modification.
\end{theorem}

We suggest the following approximate learning algorithm for a residual
network with the ReLU activation function:

\begin{description}
 \item[1.] Fix $\sigma^2_v$ and $\sigma^2_w$ and $\alpha < 1$.
 \item[2.] Given a training example $\left({\bm{y}}_t, {\bm{x}}_t
 \right)$, calculate the back-propageted error $e_{i_m}$ based on equation like (\ref{eq:am7620180619}) for each unit $i_m$ of the $m$-th layer. 
 \item[3.] Using the current $\bm{w}^{\ast}_{i_m}$,
	    calculate $\bar{A}_{\infty}, X, Y$ and $Z$ from
	    equations (\ref{eq:am7020180619})--(\ref{eq:am7220180619}) and
	    Appendix II.
 \item[4.] Update the current ${\bm{w}}^{\ast}_{i_m}$ by
	    using equations (\ref{eq:am6720180509})--(\ref{eq:am6820180509}).
 \item[5.] We may use the Polyak averaging (Polyak \& Juditsky, 1992) after
learning.  
\end{description}

It is interest to compare the result of the present algorithm with that of the quasi-diagonal method (Marceau-Caron \& Ollivier, 2016).  Ollivier (2018) proposed an adaptive method of obtaining ${\bm{G}}^{-1}
\nabla_W l$ recursively from data.  The method, called TANGO, looks promising
for implementing natural gradient learning, and our method can be used to
obtain the initial value of the velocity vector in TANGO.

\section{Conclusions}

The paper applies the statistical neurodynamical method for studying the
Fisher information matrix of deep random networks. In continuation to
the accompanying paper (Amari, Karakida \& Oizumi; 2018) on the
feedforward paths where signal are transformed through random networks
(see Poole et al, 2016), it addresses the feedback paths in which error
signals are back-propagated (Schoenholz et al., 2017). The Fisher
information is calculated from back-propagated errors.  The main
result of the paper is to prove that the Fisher information matrix in a large random network is unit-wise block diagonalized approximately.
This justifies the unit-wise natural gradient method (Ollivier, 2015).  The unit-wise Fisher information is a tensor product of the Fisher information matrices of single neurons.  We calculated the Fisher information and its inverse explicitly, showing its peculiar structure.  It justifies the quasi-diagonal natural gradient method (Ollivier, 2015).  We finally proposed to apply the present results to a resnet, where $\alpha<1$ and $v_{ij}$ are fixed.

\section*{Appendix I: \; Self-averaging property}

Let us treat a simple case
\begin{equation}
 u = \sum w_i x_i,
\end{equation}
where $x_i$ are fixed and $w_i \sim N \left(0, \sigma^2/n\right)$.  We
consider 
\begin{equation}
 E \left[ f(u) w_i w_j \right],
\end{equation}
where we put $f(u)= \left\{ \varphi'(u)\right\}^2$ and $i_{l-1}= i,
i'_{l-1}=j$.  Put
\begin{equation}
 \tilde{u} = \sum_{i \ne 1, 2} w_i x_i.
\end{equation}
Then
\begin{eqnarray}
 u &=& \tilde{u} + w_1 x_1 + w_2 x_2, \\
 f(u) &=& f \left(\tilde{u}\right) + f' \left(\tilde{u}\right)
 \left(w_i x_i + w_j x_j \right).
\end{eqnarray}
We have, neglecting higher-order terms,
\begin{equation}
 E \left[ f(u) w_i w_j \right] = E 
 \left[ f \left(\tilde{u}\right) w_i w_j \right] 
 + E \left[ f' \left(\tilde{u}\right) w_i w_j
 \left(w_i x_i + w_j x_j \right)\right].
\end{equation}
Since $\tilde{u}$ and $w_1 w_2$ are independent,
\begin{eqnarray}
 E \left[f(u)w_i w_j \right] &=& E \left[f \left(\tilde{u}\right)\right]
 E \left[w_i w_j \right] \\
 &=& E \left[f(u)\right]E \left[w_i w_j \right]
\end{eqnarray}
except for higher-order terms.

\section*{Appendix II: \; Fisher information for ReLU}

The ReLU activation function is given by
\begin{equation}
 \varphi(u) = \left\{
  \begin{array}{ll}
   u, & u>0, \\
   0, & u \le 0.
  \end{array}
 \right.
\end{equation}
We calculate $A_{ij} \left({\bm{w}}, w_0 \right)$ given by
equations (\ref{eq:am5020180511})--(\ref{eq:am5220180511}).  Since
\begin{equation}
 \varphi' \left(w x^{\ast}_1 + w_0 \right) = 
 \left\{
  \begin{array}{ll}
   1, & w x^{\ast}_1 + w_0 >0, \\
   0, & \mbox{otherwise},
  \end{array}
 \right.
\end{equation}
we have
\begin{eqnarray}
 A_{00} &=& \frac 1{\sqrt{2 \pi}} \int^{\infty}_{-\frac{w_0}w}
 \exp \left\{-\frac{u^2}2 \right\} du = \mbox{erf}
 \left( \frac{w_0}w\right), \\
 \mbox{erf}(u) &=& \frac 1{\sqrt{2 \pi}} \int^u_{-\infty}
  \exp \left\{-\frac{u^2}2\right\} dv.
\end{eqnarray}
Similarly,
\begin{eqnarray}
 A_{0n} &=& \frac 1{\sqrt{2 \pi}} \int^{\infty}_{-\frac{w_0}w}
 u \exp \left\{-\frac{u^2}2 \right\} du = 
 \frac 1{\sqrt{2 \pi}} \exp \left\{
 -\frac 12 \left( \frac{w_0}w \right)^2 \right\} \\
 A_{nn} &=& \frac 1{\sqrt{2 \pi}} \int^{\infty}_{-\frac{w_0}w} u^2 \exp
 \left\{ -\frac{u^2}2 \right\} du \\
 &=& \mbox{erf} \left( \frac{w_0}w \right) - \frac 1{\sqrt{2 \pi}}
 \frac{w_0}w \exp \left\{ -\frac 12 \left(\frac{w_0}w \right)^2 \right\}.
\end{eqnarray}

\end{document}